\documentclass{article}

\PassOptionsToPackage{numbers, compress, sort}{natbib}

\usepackage[preprint]{neurips_2021}

\usepackage[utf8]{inputenc} 
\usepackage[T1]{fontenc}    
\usepackage{hyperref}       
\usepackage{url}            
\usepackage{booktabs}       
\usepackage{amsfonts}       
\usepackage{nicefrac}       
\usepackage{microtype}      
\usepackage{xcolor}         

\usepackage{algorithm}
\usepackage[noend]{algpseudocode}
\usepackage{graphicx}
\usepackage{multirow}
\usepackage{booktabs}
\usepackage{color}
\usepackage{bm}
\usepackage{bbm}
\usepackage{amsmath}
\usepackage{amsthm}
\usepackage{amssymb}
\usepackage{subfigure}
\usepackage{pdfpages}
\usepackage{natbib}
\usepackage{wrapfig}
\usepackage{comment}
\usepackage{multicol}
\usepackage{hyperref}
\usepackage{mathtools}
\usepackage{url}
\usepackage{fontawesome5}
\usepackage{enumerate}
\usepackage{enumitem}
\usepackage{titlesec}
\usepackage{array}

\titlespacing{\paragraph}{%
  0pt}{
  0.4em}{
  0.5em}

\hypersetup{colorlinks,citecolor=blue}
\definecolor{blue}{HTML}{15316E}
\definecolor{red}{HTML}{800000}

\newtheorem{theorem}{Theorem}
\newtheorem{definition}[theorem]{Definition}
\newtheorem{lemma}[theorem]{Lemma}
\newtheorem{proposition}[theorem]{Proposition}

\newtheorem{example}[theorem]{Example}

\newcommand{\E}{\mathbb{E}}

\newcommand{\argmin}{\mathrm{argmin}}

\newcommand{\tsf}[1]{{\small\textsf{#1}}}

\allowdisplaybreaks

\title{
   Learning on Random Balls is Sufficient\\for Estimating (Some) Graph Parameters
}
\date{}

\author{
    Takanori Maehara\\
    Facebook AI\\
    London, United Kingdom\\
    \texttt{tmaehara@fb.com} \\
    \And
	Hoang NT\\
	Tokyo Tech \& RIKEN AIP\\
	Tokyo, Japan\\
	\texttt{hoangnt@net.c.titech.ac.jp}
}

\begin{document}

\maketitle

\begin{abstract}
Theoretical analyses for graph learning methods often assume a complete observation of the input graph. 
Such an assumption might not be useful for handling any-size graphs due to the scalability issues in practice. 
In this work, we develop a theoretical framework for graph classification problems in the partial observation setting (i.e., subgraph samplings). 
Equipped with insights from graph limit theory, we propose a new graph classification model that works on a randomly sampled subgraph and a novel topology to characterize the representability of the model.
Our theoretical framework contributes a theoretical validation of mini-batch learning on graphs and leads to new learning-theoretic results on generalization bounds as well as size-generalizability without assumptions on the input. 
\end{abstract}

\section{Introduction}
\label{sec:intro}

Going beyond regular structural inputs such as grids (images), sequences (time series, sentences), or general feature vectors is an important research direction of machine learning and computational sciences.
Arguably, most interesting objects and problems in nature can be described as graphs~\citep{lovasz2012large}.
For such reason, graph learning methods, especially Graph Neural Networks (GNN)~\citep{scarselli2008graph}, have recently proven to be a useful solution to many problems in computer vision~\cite{chebynets,valsesia2018learning,fey2018spline,kampffmeyer2019rethinking}, complex network analyses~\cite{graphsage,gcn,pinsage}, molecule modeling~\cite{mahe2009graph,mayr2016deeptox,gilmer2017neural,klicpera2020directional}, and physics simulations~\cite{battaglia2016interaction,kipf2018neural,pfaff2020learning}.

The significant value of graph learning models in practice has inspired a large amount of theoretical work dedicated to exploring their representational limits and the possibilities of improving them.
Most notably, the representational capability of GNNs has been in the spotlight of recent years.
To answer the question ``\emph{Can GNNs approximate all functions on graphs?}'', researchers discussed universal invariant and equivariant neural networks~\citep{keriven2019universal,maron2019provably,maehara2019simple,hoang2020graph} as \emph{theoretical} upper limits for neural architectures or showed the correspondence between message-passing GNNs (MP-GNNs) to the Weisfeiler-Leman (WL) algorithm~\cite{weisfeiler1968reduction} as \emph{practical} upper limits~\cite{morris2019weisfeiler,gin}.

Given an extremely large graph as an input, it is often impractical to keep the whole graph in the working memory.
Therefore, practical graph learning methods often utilize neighborhood samplings~\cite{graphsage,pinsage} or random walks~\citep{perozzi2014deepwalk} to handle this scalability issue. 
Because existing analyses assumed a complete observation of the input graphs~\cite{sato2020survey,keriven2019universal,hoang2020graph}, it is unclear what can be learned if we combine graph learning models with random samplings.
Thus, the relevant question in this scenario is ``\emph{What graph functions are representable by GNNs when we can only observe random neighborhoods?}.''
This question adds another dimension to the discussion of GNN expressivity; even if we have a powerful GNN (in both theoretical and practical senses), what kind of graph functions can we learn if the input graphs are too large to be computed as a whole?

\paragraph{Contributions} 
This study proposes a theoretical approach to address graph learning problems on large graphs by identifying a novel topology of the graph space.
We discuss the graph classification problem in the main part of the paper and 
extend the discussion to the vertex classification problem in Appendix~\ref{appendix:vertex}.
The extension to vertex classification can be realised by viewing it as a rooted-graph classification problem.
We first introduce a \emph{random ball sampling GNN (RBS-GNN)}, 
which is a mathematical model of GNNs implementable in a \emph{random neighborhood} computational model,
and prove that the model is universal in the class of estimable functions (Theorem~\ref{thm:universal-rbs-gnn}).
Our main contribution is introducing \emph{randomized Benjamini--Schramm topology} in the space of all graphs
and identifying the estimability of the function as the uniform continuity in this topology (Theorem~\ref{thm:estimability-equiv-continuity}).
By applying our main theorem, we obtain the following learning-theoretic results.
\begin{itemize}[noitemsep,topsep=0pt]
    \item We show the equivalent of estimability and continuity (Theorems~\ref{thm:universal-rbs-gnn} and \ref{thm:estimability-equiv-continuity}).
    This implies the continuity assumption is a sufficient condition for the mini-batch learning on graphs.
    \item 
    We prove that the functions representable by RBS-GNNs are generalizable
    by showing an upper bound of the Rademacher complexity of Lipschitz graph functions (Theorem~\ref{thm:rademacher-complexity}).
    \item We identify size-generalizable functions with estimable functions (Theorem~\ref{thm:est-size-generalizable}).
    Then, by recognizing the size-generalization as a domain adaptation,
    we provide a size-generalization error based on the Wasserstein distance  (Theorem~\ref{thm:domain-adaptation}).
\end{itemize}
Unlike existing studies, which assumed a random graph model~\cite{kawamoto2018mean,keriven2020convergence} or boundedness~\cite{sato2019approximation,gntk,keriven2019universal,hoang2020graph}, our framework does not assume anything about the graph class; instead, we assume the continuity of the graph functions.
Our results listed above are model-agnostic, i.e., we only discuss the property of the function space, regardless of how GNN models are implemented.
The model-agnostic nature of our results gives a systematic view to general graph parameters learning; their generality is especially useful as there are many different GNN architectures in practice~\cite{wu2020comprehensive,zhou2020graph,sato2020survey}.

\section{Related Work}
\label{sec:related}

\paragraph{Large-scale GNNs}
The success of GNNs, especially vertex classification models like GCN~\cite{gcn} and GraphSAGE~\cite{graphsage}, has led to various large-scale industrial GNN systems~(see \cite{abadal2020computing} and references therein).
Aiming to increase computational throughput while maintaining the predictive performance, most of these systems implemented fixed-size neighborhood sampling~\cite{graphsage} to enable large-scale batching~\cite{pinsage,zhang2020agl,zheng2020distdgl,zhihao2020improving}.
GNNs have also been applied to the 3D point clouds classification problem~\citep{wang2019gcnn}, which translates a computer vision problem to the large graph classification problem, 
and the random sampling was empirically shown to be effective~\citep{lang2020samplenet}.
In this context, our work contributes a theoretical justification for the random sampling procedure.

\paragraph{Graph Parameter Learning}
Graph function, graph parameter, or graph invariant refer to a (real or integer value) property of graphs, which only depends on the graph structure. 
In other words, they are functions defined on isomorphism classes of graphs~\cite{lovasz2012large}.
Determining graph properties from data has long been a topic of interest in theoretical computer science~\citep{lovasz2012large,czumaj2019testable} 
and is an important machine learning task in computational chemistry~\cite{debnath1991structure,gilmer2017neural} and biology~\cite{gartner2003survey,borgwardt2005protein}. 
Recently, GNNs have been proven successful on a wide range of graph learning benchmark datasets.
Current literature analyzed their expressivity to gain a better understanding of the architectures~\cite{maron2018invariant,keriven2019universal,hoang2020graph}.
Several works identify MP-GNNs to the 1-dimensional WL isomorphism test~\cite{gin} and further improve the GNN architectures to more expressive variants such as $k$-dimensional WL~\cite{morris2019weisfeiler}, port-numbered message passing~\cite{sato2019approximation}, and sparse WL~\cite{morris2020weisfeiler}. 
GNNs are also linked to the representational power of logical expressions~\cite{barcelo2019logical}.
These theoretical results assumed the complete observation of the input graph; therefore, it is difficult to see to what extent these results would hold when the only partial observation is available.
By studying the RBS-GNN model, we give an answer to this issue.
We use GNNs because they are the most expressive graph learning methods~\citep{gin,keriven2019universal}.
Nonetheless, our results generalize for other universal (Theorem~\ref{thm:universal-rbs-gnn}) and partially-universal (Theorem~\ref{thm:indist}) methods.

\paragraph{Generalization}
Besides expressivity, another challenge in graph learning is to understand the generalization bounds.
\citet{scarselli2018vapnik} introduced an upper bound for the VC-dimension of functions computable by GNNs, in which the output is defined on a special supervised vertex.
\citet{garg2020generalization} derived tighter Rademacher complexity bounds for similar MP-GNNs by considering the local computational tree structures.
\citet{liao2021pac} obtained a generalization gap of MP-GNNs and GCNs~\cite{gcn} using PAC-Bayes techniques.
\citet{gntk} obtained a sample complexity using a result in the kernel method for their graph neural tangent kernel model in learning propagation-based functions.
\citet{verma2019stability} obtained a generalization gap of single-layer GCNs by analyzing the stability and dependency on the largest eigenvalue of the graph;
\citet{lv2021generalization} derived a Rademacher bound for a similar GCN model with a similar dependency.
\citet{keriven2020convergence} assumed an underlying random kernel (similar to graphons~\cite{lovasz2012large}) and analyzed the stability of discrete GCN using a continuous counterpart c-GCN.
They derived the convergence bounds by looking at stability when diffeomorphisms~\cite{mallat2012group} are applied to the underlying graph kernel, the distribution, and the signals. 
All these methods placed some assumptions on the graph space; either bounded degree~\cite{garg2020generalization,liao2021pac,lv2021generalization}, bounded number of vertices~\cite{scarselli2018vapnik,verma2019stability,gntk}, or graphs belong to a random model~\cite{keriven2020convergence}.
Therefore, all these results become either inapplicable or unbounded in the general graph space.
Our Theorem~\ref{thm:rademacher-complexity} contributes a complexity bound without assumptions on the graphs.

\paragraph{Property Testing and Constant-Time Local Algorithms}
Property testing on graphs is a task to identify whether the input graph satisfies a graph property $\Pi$ or $\epsilon$-far from $\Pi$~\cite{goldreich2010introduction}.
Often a researcher in this area tries to derive an algorithm whose complexity is constant (i.e., only depends on $\epsilon$) or sublinear in the input size~\cite{rubinfeld2011sublinear}.
Several graph properties admit sub-linear (or constant-time) algorithms;
the examples include bipartite testing, triangle-free testing, edge connectivity, and matching~\citep{yoshida2009improved,nguyen2008constant}.
Recently, by bridging the constant-time algorithms and the GNN literature, \citet{sato2019constant} showed that, for each vertex, the neighborhood aggregation procedure of a GNN layer (they called it ``node embedding'') can be approximated in constant time.
However, this does not result in a constant-time learning algorithm for GNNs because we still need to access all the vertices to get the desired outputs.
Our results provide the first ``fully constant-time'' GNNs in the sense that the whole learning and prediction process runs in time independent of the size of the graphs (Section~\ref{sec:rbsgnn}).

\paragraph{Statistical Theory for Network Analysis}

Learning graph property from samples is a traditional topic in statistical network analysis~\cite{frank1977estimation,goodman1949estimation,kolaczyk2014statistical}.
\citet{klusowski2020estimating} proved that it is hard to estimate the number of connected components using sublinear-size samples. \citet{bhattacharya2020motif}
showed that the Horvitz--Thompson estimator for the number of subgraphs of constant size is consistent and asymptotic normal if the fourth-moment condition holds.
Our result is consistent with these results as the number of connected components is non-continuous and the number of motifs is continuous in the randomized Benjamini--Schramm topology.
These studies indicate a future direction of this study; for example, the asymptotic normality and consistency of RBS-GNNs.

\section{Preliminaries}
\label{sec:pre}

\subsection{Graphs}
\label{ssec:rooted-graph}

A \emph{(directed) graph} $G$ is a tuple $(V, E)$ of the set of vertices $V$ and the set of edges $E \subseteq V \times V$.
We use $V(G)$ for $V$ and $E(G)$ for $E$ when the graph is unclear from the context.
A graph is \emph{weakly connected} if the underlying undirected graph has a path between any two vertices.
A \emph{weakly connected component} is a maximal weakly connected subgraph.
Two graphs $G$ and $H$ are \emph{isomorphic} if there is a bijection $\phi: V(G) \rightarrow V(H)$ such that $(\phi(u),\phi(v)) \in E(H)$ if and only if $(u,v) \in E(G)$.
Let $\mathcal{G}$ be the set of all directed graphs.
A \emph{ball of radius $r$ centered at $v$}, $B_r(v)$ (also simply $B$), is the set of vertices whose shortest path distance from $v$ is bounded by $r$.
For $U \subseteq V(G)$, $G[U]$ is the subgraph of $G$ induced by $U$. 

A \emph{rooted graph} $(G, v)$ is a graph $G$ augmented with a vertex $v$ in $V(G)$.
The isomorphism between $(G,v)$ and $(H,u)$ is defined in the same way as for graphs with the extra requirement that it maps $v$ to $u$.
A \emph{$k$-rooted graph} $(G, v_1, \dots, v_k)$ is defined similarly.
We often recognize the graph induced by the ball $B_r(v)$ as a rooted graph whose root is $v$ and by the union of $k$ balls as a $k$-rooted graph.

Modern graph learning problems ask for a function $p: \mathcal{G} \rightarrow \mathcal{D}$ from training data, where $\mathcal{D}$ is a ``learning-friendly'' domain such as the set of real numbers $\mathbb{R}$, a $d$-dimensional real vector space $Q \subseteq \mathbb{R}^d$, or some finite sets.
In most cases, the function $p$ is required to be isomorphism-invariant (or invariant for short).
This notion of graph functions coincides with the definition of \emph{graph parameters}.
Another term used in the literature is \emph{graph property}, which can be formalized as a graph function whose co-domain is $\{0,1\}$. 
Our work focuses on the case in which the co-domain is $\mathbb{R}$.
 
\subsection{Computational Model}
\label{ssec:computational-model}

Extremely large graphs are usually stored in some complicated storage.
Thus, there are some constraints on how we can access the graphs.
In the area of property testing, such a situation is modeled by introducing a \emph{computational model}, which is an oracle for accessing the graph.
Importantly, each computational model induces a topology on the graph space.
As we will show in later sections, the ability to represent graph functions is related to this topology. 

There are three main computational models in the literature: the \emph{adjacency predicate model}~\citep{goldreich1998property}, the \emph{incidence function model}~\citep{goldreich1997property}, and the \emph{general graph model}~\citep{parnas2002testing,kaufman2004tight}. 
The adjacency predicate model, also known as the dense graph model, allows randomized algorithms to query whether two vertices are adjacent or not.
With the incidence function model, also known as the bounded-degree graph model, algorithms can query a specific neighbor of a vertex.
The general graph model lets the algorithms ask for both a specific neighbor and for whether two vertices are adjacent; hence, this is the most realistic model for actual algorithmic applications~\citep{goldreich2010introduction}.

In this study, we consider the following \emph{random neighborhood model}, which allows us to access the input graph $G$ via the following queries:
\begin{itemize}[noitemsep,topsep=0pt]
    \item \tsf{SampleVertex}($G$): Sample a vertex $u \in V$ uniformly randomly.
    \item \tsf{SampleNeighbor}($G$, $u$): Sample a vertex $v$ from the neighborhood of $u$ uniformly randomly, where $u$ is an already obtained vertex.
    \item \tsf{IsAdjacent}($G$, $u$, $v$): Return whether the vertices $u$ and $v$ are adjacent, where $u$ and $v$ are already obtained vertices.
\end{itemize}

This model is a randomized version of the general graph model.
\citet{czumaj2019testable} proposed a similar model to analyze edge streaming algorithms for property testing.
However, their model does not have the \tsf{IsAdjacent} query, i.e., it is a randomized version of the incidence function model.
Note that the computational model naturally specifies the \emph{estimability} of the graph parameters.
More formally, we have the following definition of estimability with respect to our random neighborhood model.
\begin{definition}[Constant-Time Estimable Graph Parameter]
\label{def:est}
A graph parameter $p$ is constant-time estimable on the random neighborhood model (estimable for short) if for any $\epsilon > 0$ there exists an integer $N$ and a randomized algorithm $\mathcal{A}$ in the random neighborhood model such that $\mathcal{A}$ performs at most $N$ queries and $|\mathcal{A}(G) - p(G)| < \epsilon$ with probability at least $1 - \epsilon$ for all graphs $G \in \mathcal{G}$. 
\end{definition}
Some examples of (non-)estimable graph parameters are:

\begin{example}
\label{example:nonest}
The number of vertices, min/max degree, and connectivity are not estimable.
\end{example}

\begin{example}
\label{example:est}
The triangle density and the local clustering coefficient are estimable.
\end{example}

Additional examples of estimable graph parameters and experimental results are provided in Appendix~\ref{appendix:exp}. 
In the next section, we implement a GNN following the proposed random neighborhood computational model.
By showing the connection between the GNN and algorithms in the random neighborhood model, we obtain several theoretical results in Section~\ref{sec:theory}. 

\section{Random Balls Sampling Graph Neural Networks (RBS-GNN)}
\label{sec:rbsgnn}

This section introduces \emph{RBS-GNN}, a theoretical GNN architecture based on the random neighborhood model.
RBS stands for ``Random Balls Sampling'' and also ''Random Benjamini--Schramm'' because our random neighborhood model extends the topology of the Benjamini--Schramm convergence~\citep{benjamini2011recurrence}.
Given an input graph, an RBS-GNN samples $k$ random vertices and proceeds to sample random balls $B_1,\dots,B_k$ rooted at each of these vertices.
A random ball of radius $r$ and branching factor $b$ is a subgraph obtained by the procedure \tsf{RandomBallSample}, illustrated in Figure~\ref{fig:rbs_illustration} for $r=1$ and $b=4$. 
The exact procedure is presented in Algorithm~\ref{algo:practical-rbsgnn}. 
It is trivial to see that \tsf{RandomBallSample} can be implemented under the random neighborhood model with \tsf{SampleVertex} and \tsf{SampleNeighbor}.

After sampling $k$ random balls, the next step is identifying the induced subgraph $G[B_1 \cup \dots \cup B_k]$ using the \tsf{IsAdjacent} procedure and computing the weakly connected components $C_1, \dots, C_{N_{C}}$ of the induced subgraph 
(Step 3 of Figure~\ref{fig:rbs_illustration}).
The classifier part of an RBS-GNN has two trainable components: a multi-layer perceptron $g$ and a GNN $f$.
The output of an RBS-GNN is defined as
\begin{align}
    \label{eq:rb-gnn}
    \text{RBS-GNN}(G) = g\left(\sum_j f(C_j)\right).
\end{align}
It should be emphasized that, as mentioned in the end of Section~\ref{sec:related}, our RBS-GNN can be evaluated in constant time (i.e., only dependent on the hyperparameters) because the subsets returned by \tsf{RandomBallSample} has a constant size regardless of the size of input graph $G$.

\begin{figure}
    \centering
    \includegraphics[width=0.98\textwidth]{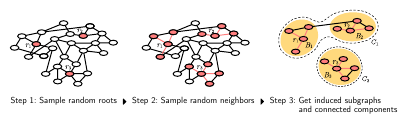}
    \vspace{-1em}
    \caption{Random Balls Sampling Procedure (Algorithm~\ref{algo:practical-rbsgnn}). Our computational model is different from the existing general graph model at Step 2, where we sample neighbors randomly instead of taking all neighbors. In Step 3, the randomly sampled edges are shown with color, and the induced edges are black. The weakly connected components $C_1$ and $C_2$ are inputs to the GNN.}
    \label{fig:rbs_illustration}
\end{figure}

\begin{algorithm}
\caption{Randomized Benjamini--Schramm GNN}
\label{algo:practical-rbsgnn}
\begin{algorithmic}[1]
\Procedure{RandomBallSample}{$G$, $b$, $r$}
    \State $\texttt{layer}[0] \leftarrow []$, $\dots$, $\texttt{layer}[r]$ $\leftarrow$ $[]$
    \State Sample one random vertex from $V(G)$ and insert to $\texttt{layer}[0]$
    \For{$i = 1, \dots, r$}
    \For{$u$ in $\texttt{layer}[i-1]$}
    \State Sample $b$ random vertices (with replacement) from $\mathcal{N}(u)$ and insert to $\texttt{layer}[i]$
    \EndFor
    \EndFor
    \State{\Return $G[\texttt{layer}[0] \cup \dots \cup \texttt{layer}[r]]$}
\EndProcedure

\Procedure{RBS-GNN}{$G, f, g, b, r, k$} 
    \State $B_1, \dots, B_k \leftarrow$ \tsf{RandomBallSample}($G, b, r$) \Comment{Runs $k$ times to get $k$ balls.}
    \State $C_1, \dots, C_{N_C} \leftarrow$ \tsf{WeaklyConnectedComponents}($G[B_1\cup \dots B_k]$) 
    \State {\Return $g(\sum_j f(C_j))$}
\EndProcedure
\end{algorithmic}
\end{algorithm}

\paragraph{Relation to Existing GNN Models}
While RBS-GNN is motivated by the random neighborhood model, it has a strong connection with existing message-passing GNNs and optimization techniques in graph learning.
When we select $f$ to be a simple message-passing GNN, RBS-GNN is a generalization of the mini-batch version of GraphSAGE (Algorithm~2 in \citep{graphsage}), and the multi-layers perceptron module $g$ acts as the global \tsf{READOUT} as in the GIN~\cite{gin} architecture.
Our analysis technique also applies for other mini-batch approaches of popular graph learning models~\citep{gcn,gat,dgi}.
This result provides a theoretical validation for the mini-batch approach in practice.
Note that this is a positive result for continuous graph parameters, and we make no claim for the non-continuous case.
On the other hand, $f$ can also be a more expressive variant such as high-order WL~\citep{morris2020weisfeiler}, $k$-treewidth homomorphism density, or a universal approximator~\citep{keriven2019universal,hoang2020graph}.


\section{Main Result}
\label{sec:theory}

In this section, we conduct theoretical analyses of RBS-GNN for the graph classification problem.
All the proofs are in Appendix~\ref{appendix:proofs}.
To simplify the analysis, we assume the hyperparameters $k$, $b$, and $r$ have the same value, and by a slight abuse of notation, we denote these values by $r$.
Note that this setting would not alter the notion of estimability.
We further simplify the discussion by assuming the graphs have no vertex features.
Similar results hold when the vertices have finite-dimensional vertex features; see Appendix~\ref{appendix:feat}.
Additionally, as mentioned in Section~\ref{sec:intro}, we obtained complementary results for the vertex classification problem in Appendix~\ref{appendix:vertex}.

\subsection{Universality of RBS-GNN}
\label{ssec:universality}

We first characterize the expressive power of RBS-GNN.
The following shows the universality of RBS-GNN, with a universal GNN component $f$, in the space of the estimable functions.

\begin{theorem}[Universality of RBS-GNN] 
\label{thm:universal-rbs-gnn} 
If a graph parameter $p: \mathcal{G} \rightarrow \mathbb{R}$ is estimable (in the random neighborhood model), then it is estimable by an RBS-GNN  with a universal GNN $f$.
\end{theorem} 

The proof of this theorem is an adaptation of the proof techniques by \citet{czumaj2019testable}.
We first introduce a \emph{canonical estimator}, which is an algorithm in the random neighborhood model defined by the following procedure.
(1) Sample $r$ random balls $B_1, \dots, B_r$ using \tsf{RandomBallSample}($G$, $r$, $r$);
(2) Return a number according to the isomorphism class of the subgraph $G[B_1 \cup \dots \cup B_r]$ induced by the balls.
Since the number of random balls, the branching factor, and the radius are constant, we can see that the size of $G[B_1 \cup \dots \cup B_r]$ is bounded by $r^{r+2}$.
Therefore, we can list all isomorphism classes of all graphs having at most $r^{r+2}$ vertices and assign a unique number to each of them.
Also, since the induced subgraph is bounded, it is possible to construct a universal approximator GNN~\citep{keriven2019universal,hoang2020graph}.
Therefore, we obtained the following.

\begin{lemma}
\label{lem:canonical-estimable}
If a graph parameter $p$ is estimable, then it is estimable by a canonical estimator.
\end{lemma}

Since a canonical estimator assigns a number according to the isomorphism class of the input, we see that RBS-GNN can approximate the canonical estimator by letting $f$ be a universal approximator for bounded graphs.
See the proof in Appendix~\ref{appendix:proofs-rbsgnn} for more detail.


\paragraph{Relation to Universality Results}
Existing universal GNNs assumed that the number of vertices of the input graphs are bounded.
Theorem~\ref{thm:universal-rbs-gnn} shows that these universal GNNs for bounded graphs can be extended to general graphs by approximating the general graphs using the random balls sampling procedure. 
As a drawback, the theorem is only applicable to the continuous functions in the randomized Benjamini--Schramm topology introduced below.
We emphasize that this drawback shows the limitation of the partial-observation (random neighborhoods) setting.

\subsection{Topology of Graph Space: Estimability is Uniform Continuity}
\label{ssec:topology}

The previous section defined the estimability by the existence of an estimation algorithm.
Such definition is suitable for algorithmic analysis; however, it is not suitable for further analysis, such as deriving the generalization error bounds.
This section rephrases our estimability by the continuity in a new topology induced by a distance between two graphs.

We start with a simple example in Figure~\ref{fig:zprofile_example}.
The figure shows an input graph $G$ consisting of two isolated vertices and one edge.
For simplicity, Algorithm~\ref{algo:practical-rbsgnn} only samples a single ball ($k=1$) with radius two ($r=2$).
This configuration let us obtain two isomorphism classes with equal probability: a single vertex and a single edge.
Hence, the event which each isomorphism class is obtained from Algorithm 1 can be represented by a two-dimensional vector.
We denote this vector $Z_2(G)$, which takes value $[0,1]$ when the single vertex is sampled and value $[1,0]$ when the single edge is sampled.
Clearly, this is a random vector whose expectation defines a distribution over the isomorphism classes with respect to the configuration of Algorithm 1.
More generally, we can define the $r$-profile $Z_r(G)$ and the corresponding isomorphism class distribution $z_r(G)$ for positive integer values of radius $r$.

\begin{figure}[t]
    \centering
    \includegraphics[width=\textwidth]{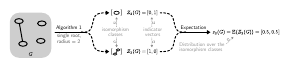}
    \vspace{-2em}
    \caption{An illustrating example for the $r$-profile with $k=1, r=2$, and a simple graph $G$.}
    \label{fig:zprofile_example}
\end{figure}

For an integer $r$, an \emph{$r$-profile} $Z_r(G)$ of a graph $G$ is a random variable of the ($k$-rooted\footnote{For simplicity, we let $k=r$.}) isomorphism class of $G[B_1 \cup \dots \cup B_r]$, where each $B_j$ is obtained from \tsf{RandomBallSample}($G$, $r$, $r$).
As \tsf{RandomBallSample}($G$, $r$, $r$) produces a graph of size at most $r^r$, we can identify $Z_r(G)$ as a random finite-dimensional vector.
Let $z_r(G) = \E[ Z_r(G) ]$ be the probability distribution over the isomorphism classes in terms of the $k$-rooted graph isomorphism, where the expectation is taken over \tsf{SampleVertex} and \tsf{SampleNeighbor}.
The \emph{sampling distance} of two graphs is defined by
\begin{align}
    d(G, H) = \sum_{r=1}^\infty 2^{-r} d_{TV}(z_r(G), z_r(H)),
\end{align}
where $d_{TV}$ is the total variation distance of two probability distributions given by $d_{TV}(p, q) = (1/2) \| p - q \|_1$.
It should be emphasized that the sampling distance allows us to compare any two graphs even though they have a different number of vertices.
We call the topology on the set of all graphs $\mathcal{G}$ induced by this sampling distance \emph{randomized Benjamini--Schramm topology}.
A graph parameter is defined to be \emph{uniformly continuous in the randomized Benjamini--Schramm topology} by the followings.

\begin{definition}[Uniform Continuity in the RBS Topology]
\label{def:uniform-continuous}
A graph parameter $p: \mathcal{G} \rightarrow \mathbb{R}$ (resp. a randomized algorithm $\mathcal{A}$) is uniformly continuous if for any $\epsilon > 0$ there exists $\delta > 0$ such that for any $G$ and $H$, $d(G, H) \le \delta$ implies $|p(G) - p(H)| < \epsilon$ (resp. $|\mathcal{A}(G) - \mathcal{A}(H)| \le \epsilon$ with probability at least $1 - \epsilon$).
\end{definition}
This topology connects the estimability in terms of the continuity as follows.

\begin{theorem}
\label{thm:estimability-equiv-continuity}
A graph parameter $p$ is estimable in the random neighborhood model if and only if it is uniformly continuous in the randomized Benjamini--Schramm topology.
\end{theorem}
The ``if'' direction of this theorem is given by the triangle inequality and the (optimal) coupling theorem~\citep{cuestaalbertos1993optimal}; the ``only-if'' is proved using the fact that the graph space is totally bounded as follows.
\begin{lemma}[Totally Boundedness of Graph Space]
\label{lem:totally-bounded-epsilon-net}
For any $\epsilon > 0$, there exists a set of graphs $\{ H_1, \dots, H_C \}$ with 
$C \le 
    2^{
        2^{
            (\log 1/\epsilon)^{
                O(\log 1/\epsilon)
            }
        }
    }$
such that $\min_{j \in \{1, \dots, C \}} d(G, H_j) \le \epsilon$ for all $G$.
\end{lemma}

Theorem~\ref{thm:estimability-equiv-continuity} allows us to apply existing ``functional analysis techniques'' to analyze the estimable functions.
We present such applications in Section~\ref{sec:applications}.

\paragraph{Intuition of the Sampling Distance and Relation to Benjamini--Schramm Topology}

The intuition behind our definition of the sampling distance reflects the idea that two graphs are similar when random samples from these graphs look similar,
where the similarities on small samples are more important to the similarities on large samples.
%
This definition generalizes the Benjamini--Schramm topology for the space $\mathcal{G}_D$ of all graphs of degree bounded by $D$, which uses a different definition of the ``$r$-profile:''
Let us define the $r$-profile by the union of $r$ balls of radius $r$ whose centers are sampled randomly. 
Then, the sampling distance defined using this $r$-profile induces a topology called the \emph{Benjamini--Schramm topology}.
This topology was first studied by Benjamini and Schramm~\cite{benjamini2011recurrence} to analyze the planar packing problem,
and now it is widely used to analyze the limit of bounded degree graphs, where the limit object is identified as the graphing; see \cite{lovasz2006limits}.
A practical issue of the 
Benjamini--Schramm topology is that it is only applicable to bounded degree graphs, where many real-world extremely large graphs are complex networks having power-law degree distributions (i.e., unbounded degree).
We addressed this issue by introducing the randomized Benjamini--Schramm topology, which is applicable to all graphs.

\section{Theoretical Applications}
\label{sec:applications}
%
%
%

\subsection{Robustness Against Perturbation}

The continuity immediately implies the robustness against the structural perturbation, i.e., for any $\epsilon > 0$ there exists $\delta > 0$ such that
the output of RBS-GNN does not change more than $\epsilon$ if the graph is perturbed at most $\delta$ in the sampling distance.
As the perturbation in sampling distance may not be intuitive in practice, we here provide a bound regarding the additive perturbation edges.

\begin{proposition}
\label{prop:perturbation}
Let $G$ be a graph and let $G'$ be the graph obtained from $G$ by adding $\delta |V(G)|$ edges completely randomly where $0 < \delta < 1$.
Then $d(G, G') = O(1 / \log (1 / \delta))$.
\end{proposition}

This result indicates that to change the output of RBS-GNN, one needs to add linearly many random edges; it is impractical in extremely large graphs.
Note that the ``adversarial'' perturbation can change the distance more easily, especially if there is a ``hub'' in the graph;
see Appendix~\ref{appendix:proofs-app} for details.

\subsection{Rademacher Complexity}

Thus far, we only discussed the expressibility of the functions regardless of the learnability.
Here, we derive the Rademacher complexity for the class of Lipschitz functions in the random Benjamini--Schramm topology.
This gives an algorithm-independent bound of the learnability of the functions.
\begin{theorem}
\label{thm:rademacher-complexity}
Let $n$ be the number of training instances.
The Rademacher complexity $R_n$ of the set of $1$-Lipschitz functions that maps to $[0, 1]$ is $(\log \log n)^{-O(1/\log \log \log \log n)}$. 
It is $o(1/\log \log \log n)$.
\end{theorem}

This result implies that, by minimizing the empirical error of $n$ instances, we can achieve the generalization gap of $o(1/\log \log \log n)$ with high probability.
To the extent of our knowledge, this is the first Rademacher bound for the general graph space, which guarantees the asymptotic convergence on any graph learning problem without assuming any graph structure.

\paragraph{Comparison with Existing Results}
The significant difference between existing studies~\citep{scarselli2018vapnik,verma2019stability,gntk,lv2021generalization,garg2020generalization,liao2021pac} and our bound (Theorem~\ref{thm:rademacher-complexity}) is that 
ours is independent of any structural property, such as the maximum number of vertices, the maximum degree, and the spectrum of the graphs.
Thus, ours can be applied to any graph distribution.
Simply put, this is a consequence of the totally boundedness of the graph space (Lemma~\ref{lem:totally-bounded-epsilon-net}): 
For any $\epsilon > 0$, the space of all graphs is approximated by finitely many graphs; hence any graph parameter is bounded by the values among them, which is a constant depending on $\epsilon$.
The drawback of this generality is its poor dependency on the number of instances $n$, which leaves significant room for quantitative improvement.
One possible way to improve the bound is by assuming some properties of the graph distribution because the above derivation is distribution-agnostic;
a concrete strategy for improvement is left for future works.

\subsection{Size-Generalizability}

One interesting topic of GNNs is \emph{size-generalization}, which is a property that a model trained on small graphs should perform well on larger graphs.
Size-generalization is observed in several tasks~\cite{khalil2017learning}; 
however, it has also been proved that some classes of GNNs do not naturally generalize~\cite{yehudai2021from}.
Hence, we want to know about the conditions for GNNs to generalize.

We need to distinguish the ``approximation-theoretic'' size-generalizability and the ``learning-theoretic'' size-generalizability.
The former is the possibility of size-generalization, which is proved by showing the existence of size-generalizing models.
This, however, does not mean that a size-generalizable model is obtained by training; thus, we need to introduce the latter. 
The latter is the degree of size-generalizability when we train a model using a dataset (or a distribution);
it is proved by bounding the generalization error.

\subsubsection{Approximation-Theoretic Size-Generalizability}

We say that a function $p$ is \emph{size-generalizable in approximation-theoretic sense} if for any $\epsilon > 0$, there exists $N>0$ such that we can construct an algorithm $\mathcal{A}$ using dataset $\{ (p(G_{\le N}), G_{\le N}) : |V(G_{\le N})| \le N \}$ such that $|p(G) - \mathcal{A}(G)| \le \epsilon$ with probability at least $1 - \epsilon$ for all $G \in \mathcal{G}$.
This gives one mathematical formulation of the size-generalizability as it requires to fit algorithm $\mathcal{A}$ to all graphs using the dataset of bounded graphs.
In this definition, we have the following theorem.

\begin{theorem}
\label{thm:est-size-generalizable}
Estimable functions are size-generalizable in the approximation-theoretic sense.
\end{theorem}

This theorem is proved by constructing a size-generalizable algorithm.
We first pick the continuity constant $\delta$ for $\epsilon$ using Theorem~\ref{thm:estimability-equiv-continuity}.
Then, we construct a $\delta$-net using Lemma~\ref{lem:totally-bounded-epsilon-net}.
By storing all the values $p(G_i)$ for the graphs in the $\delta$-net, we obtain a size-generalizable algorithm, where $N$ is the maximum number of the vertices in the $\delta$-net.

\subsubsection{Learning-Theoretic Size-Generalizability}

From the learning theoretic viewpoint, size-generalization is a domain adaptation from the distribution of smaller graphs to the distribution of larger graphs~\citep{yehudai2021from}.
Thus, it is natural to utilize the domain adaptation theory~\citep{redko2019advances}.
Especially since we have introduced the sampling distance defined on all pairs of graphs irrelevant to their sizes, we here employ the Wasserstein distance-based approach~\citep{shen2018wasserstein}.

We start from a general situation.
Let $\mathcal{D}_1$ and $\mathcal{D}_2$ be joint distributions of graphs and their labels, and $\mathcal{G}_1$ and $\mathcal{G}_2$ be the corresponding marginal distributions of graphs.
We abbreviate $\E_1$ and $\E_2$ for the expectations on $\mathcal{D}_1$ and $\mathcal{D}_2$, respectively.
The Wasserstein distance between $\mathcal{G}_1$ and $\mathcal{G}_2$ is given by
\begin{align}
    W(\mathcal{G}_1, \mathcal{G}_2) = \inf_{\pi} \E_{(G_1, G_2) \sim \pi} d(G_1, G_2),
\end{align}
where $d$ is the sampling distance of the graphs, and $\pi$ runs over the couplings between these distributions.
Let $\lambda = \inf_h \{ \E_1 | y - h(G)| + \E_2 | y - h(G) | \}$ be the optimal combined error, where $\inf_h$ runs over all $1$-Lipschitz functions.
We have the following lemma.
\begin{lemma}
\label{lem:transfer-bound}
For any $1$-Lipschitz functions $h$ and $h'$, we have the following.
\begin{align}
    \E_1 |y - f(G)|
    \le \E_2 |y - f(G)| + 2 W(\mathcal{G}_1, \mathcal{G}_2) + \lambda,
\end{align}
\end{lemma}

Combining this result with the Rademacher complexity (Theorem~\ref{thm:rademacher-complexity}), we obtain the following generalization bound.
\begin{theorem}
\label{thm:domain-adaptation}
Let $\epsilon > 0$.
Let $(y_{21}, G_{21}), \dots, (y_{2n}, G_{2n})$ be independently drawn from $\mathcal{D}_2$. 
If $\lambda = O(\epsilon)$ and $n \ge 2^{2^{2^{\tilde\Omega(1/\epsilon)}}}$, then, for any $1$-Lipschitz function $h$, we have
\begin{align}
    \E_{(G_1, y_1) \sim \mathcal{D}_1}[ |y_1 - h_1(G)| ] 
    \le 
    \frac{1}{n} \sum_{i=1}^n | y_{2i} - h(G_{2i}) |
     + 2 W(\mathcal{D}_1, \mathcal{D}_2) + O(\epsilon)
\end{align}
with probability at least $1 - \epsilon$.
\end{theorem}
The condition $\lambda = O(\epsilon)$ requires the existence of a ``consistent rule'' among both $\mathcal{D}_1$ and $\mathcal{D}_2$.
For example, this condition holds when the labels are generated by $y = f(G) + \epsilon \mathcal{N}(0, 1)$ for some $1$-Lipschitz function $f$, where $\mathcal{N}(0, 1)$ is the standard normal distribution.
We can obtain the size-generalization bound by applying the above theorem for the distribution of large graphs $\mathcal{G}_1$ and of small graphs $\mathcal{G}_2$.
Thus, we only need to evaluate their Wasserstein distance.
The Wasserstein distance can be large in the worst-case;
thus, we here consider concrete examples of graph distributions.

First, we consider the case that undirected graphs are drawn from the \emph{configuration model of $d$-regular graphs}.
In this model, a graph is constructed by the following procedure:
(1) It creates $N$ vertices with $d$ half-edges;
(2) Then, it pairs the half-edges and connects them to obtain edges.
We see that a learning problem on this distribution is size-generalizable.
\begin{proposition}
\label{prop:size-gen}
Let $\mathcal{G}$ be a distribution of random $d$-regular graphs generated by the configuration model, and $\mathcal{G}_{\le N}$ be the distribution conditioned on only graphs of size bounded by $N$.
If $N \ge (\log 1/\epsilon)^{\Omega(\log 1/\epsilon)}$ then $W(\mathcal{G}, \mathcal{G}_{\le D}) = O(\epsilon)$.
\end{proposition}

This result can be generalized to a general distribution of graphs with large girth.
Next, we consider the case where undirected graphs are drawn from a graphon.
A \emph{graphon} $\mathcal{W}$ is a function $\mathcal{W} : [0, 1] \times [0, 1] \to [0, 1]$.
A graph $G_N$ is drawn from $\mathcal{W}$ if we first draw $N$ random numbers $x_1, \dots, x_N \in [0, 1]$ uniformly randomly.
Then, for each pairs $(x_i, x_j)$, we put an edge with probability $\mathcal{W}(x_i, x_j)$.
This model extends Erdos--Renyi random graph and stochastic block model; see \cite{lovasz2012large} for more detail.

\begin{proposition}
\label{prop:size-gen2}
Let $\mathcal{W}$ be a graphon.
Let $N_1$ and $N_2$ be integers with $N_1 < N_2$.
Let $\mathcal{G}_{N_i}$ be a distribution of graphs of $N_i$ vertices drawn from $\mathcal{W}$.
If $N_1 \ge 2^{O(1/\epsilon^2)}$ then $W(\mathcal{G}_{N_1}, \mathcal{G}_{N_2}) \le \epsilon$.
\end{proposition}

Finally, we consider the case that $\mathcal{D}_N$ is obtained from $\mathcal{D}$ by the metric projection.
Let $\Pi_N$ be the projection onto the space of graphs of size at most $N$, i.e., $\Pi(G) = \argmin_{G_N: |V(G_N)| \le N} d(G, G_N)$.
\begin{proposition}
\label{prop:wasserstein}
Let $\mathcal{G}$ be any graph distribution and let $\mathcal{G}_{\le N} = \Pi(\mathcal{G})$ be the projected distribution of graphs of size at most $N$.
For any $\epsilon > 0$, there exists $N$ such that $W(\mathcal{G}, \mathcal{G}_{\le N}) \le \epsilon$.
\end{proposition}

A drawback of this result is that an explicit bound of $N$ is not known, even for its deterministic variant in the bounded degree graphs (See Proposition~19.10 in \cite{lovasz2012large}).
The only known bound is for the bounded degree graphs with large girth~\citep{fichtenberger2015constant}. 

\paragraph{Comparison with Existing Results}
Size-generalization of GNNs is reported on several tasks, but its theoretical analysis is limited.
\citet{yehudai2021from} studied the size-generalizability using the concept of $d$-pattern, which is information obtained from $d$-ball; it is similar to our $r$-profile. 
Their results are approximation-theoretic as they showed the (non-)existence of size-generalizable models but did not show how such models can be obtained by training on data.
\citet{xu2020neural} proved the size-generalization of the max-degree function under several conditions on the training data and GNNs.
Their result is essentially an approximation-theoretic as it assumes the dataset lies in and spans a certain space that is sufficient to identify the max-degree function.

\subsection{Partially-Universal RBS-GNNs}
\label{sec:partially-universal-rbs-gnn}

Thus far, we assumed the universal GNNs are plugged into the RBS-GNNs for theoretical analysis.
This assumption achieves the maximum expressive power in this framework; 
however, in practice, we often use expressive but more efficient GNNs such as GCN~\cite{gcn}, GIN~\cite{gin}, or GAT~\cite{gat}.
Here, we discuss what will be changed if we made this modification.

Let $\equiv$ be an equivalence relation on graphs.
We assume that $\equiv$ is \emph{consistent with the weakly connected component decomposition}, i.e., 
if $G_1 \equiv H_1$ and $G_2 \equiv H_2$ then $G_1 + G_2 \equiv H_1 + H_2$, where the ``$+$'' symbol denotes the disjoint union of two graphs.
We say that a function $h$ (resp. a randomized algorithm $\mathcal{A}$) is \emph{$\equiv$-indistinguishable} if $h(G) = h(G')$ (resp. $\mathcal{A}(G) = \mathcal{A}(G')$ given the random sample) for all $G \equiv G'$.
A GNN is \emph{$\equiv$-universal} if it can learn any $\equiv$-indistinguishable functions.
For example, it is known that GIN is universal with respect to the WL indistinguishable functions~\cite{gin}.
Let RBS-GNN[$\equiv$] be a class of RBS-GNNs that uses an $\equiv$-universal GNN $f$ in Equation~\eqref{eq:rb-gnn}.
The following shows the partial universality of this architecture.

\begin{theorem}
\label{thm:indist}
If $f$ is estimable and $\equiv$-indistinguishable, then it is estimable by an RBS-GNN[$\equiv$].
\end{theorem}


One application of this theorem is extending the expressivity of GraphSAGE to the partial observation setting.
GraphSAGE can represent the local clustering coefficient if we have the complete observation of the graph~\citep[Theorem~1]{graphsage}.
We can prove the local clustering coefficient is estimable (Proposition~\ref{prop:clustering} in Appendix~\ref{appendix:exp}).
By applying Theorem~\ref{thm:indist} to the equivalence relation $G_1 \equiv G_2$ defined by $f(G_1) = f(G_2)$ for all function $f$ representable by GraphSAGE, we obtain the following.
\begin{proposition}
\label{prop:graphsage-algo2-clustering-coeff}
The mini-batch version of the GraphSAGE (Algorithm~2 in \citep{graphsage}) can estimate the local clustering coefficient.
\end{proposition}

\paragraph{Comparison with Existing Studies}
Equivalence relations associated with GNNs are mainly studied in the context of the ``limitation'' of GNNs: 
If a GNN is $\equiv$-indistinguishable, then it cannot learn any function $h$ that is non $\equiv$-indistinguishable.
\citet{morris2019weisfeiler} proved a message-passing type GNN cannot distinguish two graphs having the same $d$-patterns.
\citet{garg2020generalization} identified indistinguishable graphs of several GNNs, including GCN~\cite{gcn}, GIN~\cite{gin}, GPNGNN~\cite{sato2019approximation}, and DimeNet~\cite{klicpera2020directional}.

On the other hand, we should use non-universal GNNs in practice because more expressive GNNs have higher computational costs (e.g., universal GNNs~\citep{keriven2019universal} is more costly than the graph isomorphism test). 
We here considered RBS-GNN[$\equiv$] because 
$\equiv$-universal GNNs are theoretically tractable classes of non-universal GNNs.
With similar motivation, \citep{hoang2020graph} proposed GNNs parameterized by information aggregation pattern and proved the $\equiv$-universality, where $\equiv$ is induced by the aggregation patterns.

\section{Conclusion}
\label{sec:conclusion}

We answered the question ``\emph{What graph functions are representable by GNNs when we can only observe random neighborhoods?}''
by proving the functions representable by RBS-GNNs coincides with the estimable functions in the random neighborhood model, which is equivalent to the uniformly continuous functions in the randomized Benjamini--Schramm topology.
The word ``(Some)'' in our title was meant to emphasize the restriction to continuous graph functions.
The result holds without any assumption on the input graphs, such as the boundedness.
This result gives us a ``functional analysis view'' of graph learning problems and leads to several new learning-theoretic results.
The weakness of our result is the poor dependency on the number of training instances, which is the trade-off for generality.
We believe addressing this issue will be an interesting future direction.
Another future direction links to the asymptotic behaviors of RBS-GNNs.
Motivating example includes the motif counting problem, in which \citet{bhattacharya2020motif} proved the Horvitz--Thompson estimator is consistent and asymptotically normal.
A corresponding result for the RBS-GNN will allow us to obtain a confidence interval of an estimation and to perform statistical testing.


\paragraph{Potential Impact}
Our work contributes an understanding of general graph learning models whose inputs are random samples of arbitrarily large graphs.
Due to the theoretical nature of our results, we believe there will not be a direct nor indirect negative societal impact.

\paragraph{Acknowledgement}
We would like to thank the anonymous reviewers and the area chairs for their thoughtful comments which help us improve our manuscript.
HN is partially supported by the Japanese Government MEXT SGU Scholarship No. 205144. 

\newpage

\bibliography{main}
\bibliographystyle{plainnat}

\newpage
\begin{appendix}
\section{Complete proofs}
\label{appendix:proofs}

\subsection{Theorem~\ref{thm:universal-rbs-gnn}}
\label{appendix:proofs-rbsgnn}

Theorem~\ref{thm:universal-rbs-gnn} states the universality of the proposed RBS-GNN.
This theorem was proposed to address the question ``Even if we have a powerful GNN, what kind of graph functions can we learn if the input graphs are too large to be computed as a whole?'' posed in the Introduction.
We think of this theorem from two perspectives.
In one view, this theorem extends the universality of ``complete-observation'' GNNs in to ``partial-observation''.
In another, the theorem reduced the universality of GNNs to only universal on estimable functions.
This section presents proofs leading up to Theorem~\ref{thm:universal-rbs-gnn}.

\begin{proof}[Proof of Lemma~\ref{lem:canonical-estimable}]
We construct a canonical estimator from the original estimator.
Let $\mathcal{E}$ be the original estimator and let $N$ be the total number of queries of to achieve accuracy $\epsilon^2/2$.
We first construct an estimator $\mathcal{E}_1$.
$\mathcal{E}_1$ samples an $N$ random balls using \tsf{RandomBallSample}($G$, $N$, $N$) and simulates $\mathcal{E}$ on $\mathcal{E}_1$ using a permutation $\pi$ over the vertices of $B_1 \cup \dots \cup B_r$.
Because of the simulation, we obtain
\begin{align}
    \mathrm{Prob}_{S, \pi}\left[ |f(G) - \mathcal{E}_1(G \mid S, \pi)| > \epsilon^2/2 \right] < \epsilon^2/2.
\end{align}
Then, we construct the final estimator $\mathcal{E}_2$.
$\mathcal{E}_2$ returns the expected value of the output of $\mathcal{E}_1$ over all simulations.
Here, 
\begin{align}
    \mathbb{E}_S [ |f(G) - \mathcal{E}_2(G \mid S)] | ] 
    &=
    \mathbb{E}_S [ |f(G) - \mathbb{E}_\pi [\mathcal{E}_1(G \mid S, \pi)] | ] \\
    &\le 
    \mathbb{E}_{S,\pi} [ |f(G) - \mathcal{E}_1(G \mid S, \pi) | ] \\
    &\le 
    \epsilon^2.
\end{align}
Thus, by the Markov inequality,
\begin{align}
    \mathrm{Prob}_S [ |f(G) - \mathcal{E}_2(G \mid S)] | > \epsilon ] \le \epsilon.
\end{align}
\end{proof}

Using Lemma~\ref{lem:canonical-estimable}, we obtain the proof for Theorem~\ref{thm:universal-rbs-gnn}.

\begin{proof}[Proof of Theorem~\ref{thm:universal-rbs-gnn}]
The output of the canonical estimator is determined by the isomorphism class of the subgraph induced by the balls.
Hence, it is determined by the isomorphism classes of the weakly connected components of the induced subgraph.
This means that we can write the canonical estimator as a function from the set of weakly-connected graphs: $h(\{ C_1, \dots, C_l\})$.
Here, we can injectively map each $C_j$ as a finite-dimensional vector $z_j$ using a universal neural network $f$ since it has a size bounded by a constant (depending on $N$).
Also, the number of connected components is bounded by a constant (depending on $N$).
Thus, we can identify the function $h$ as a permutation-invariant function with a constant number of arguments whose inputs are finite-dimensional vectors.
Therefore, we can apply Theorem~9 in \cite{zaheer2017deep}, which shows that there exists continuous functions $g$ and $\phi$ such that $h(z_1, \dots, z_m) = g(\sum_i \rho(z_m))$. for all $z_1, \dots, z_m$.
Because $\phi$ is approximated by a neural network, we can combine it with the GNN $f$; therefore, we obtain the proof.
\end{proof}

\subsection{Theorem~\ref{thm:estimability-equiv-continuity}}
\label{appendix:proofs-continuity}

Theorem~\ref{thm:estimability-equiv-continuity} is perhaps the most important contribution of this work. 
This theorem continues to analyze the concept of estimable graph functions by providing a topology in which estimable functions are continuous and vice versa.
We find that the result is quite useful when we want to apply functional analysis techniques to analyze graph learning problems.
This section presents the proofs of Lemma~\ref{lem:totally-bounded-epsilon-net} and Theorem~\ref{thm:estimability-equiv-continuity}.

\begin{proof}[Proof of Lemma~\ref{lem:totally-bounded-epsilon-net}]
This is a variant of \cite[Proposition 19.10]{lovasz2012large}, which is for a different topology (different computational model). 
We can prove our lemma by the same strategy, but here we provide a proof for completeness.

Let $r = \lceil \log 2 / \epsilon \rceil $.
We choose a maximal set of graphs $H_1, \dots, H_N$ such that for all $i \neq j$, $d_{TV}(z_s(H_i), z_s(H_j)) > \epsilon / 4$ holds on some $s \le r$.
We can see that such a set exists (see below).
By the maximality, for any graph $G$, there exists $j$ such that $d_{TV}(z_s(G), z_s(H_j)) \le \epsilon/4$ for all $s \le r$, which implies $d(G, H) \le (1/2)^r + \epsilon/2 \le \epsilon$.

We show the upper bound of $C$.
In the proof, we represent $G$ by an $r$-tuple $(z_1(G), \dots, z_r(G))$ of probability distributions, where each $z_s(G)$ lies on $2^{s^s \times s^s} = 2^{(\log 1/\epsilon)^{O(\log 1/\epsilon)}}$-dimensional simplex for $s \le r$.
Because the packing number of $d$-dimensional simplex in the total variation distance (equivalently in the $l_1$ metric) is $(1 / \epsilon)^{O(d)}$,
we cannot choose more than $2^{2^{(\log 1/\epsilon)^{O(\log 1/\epsilon)}}}$ points whose pairwise distance is at least $\epsilon / 4$.
\end{proof}

\begin{proof}[Proof of Theorem~\ref{thm:estimability-equiv-continuity}]
Suppose $f$ is estimable. 
For any $\epsilon > 0$, we choose a canonical estimator $\mathcal{A}$ of accuracy $\epsilon$.
Then we have $|f(G) - f(H)| \le |f(G) - \mathcal{A}(G)| + |\mathcal{A}(G) - \mathcal{A}(H)| + |\mathcal{A}(H) - f(H)|$.
Here, the first and last terms are at most $\epsilon$ by the definition of $\mathcal{A}$ with probability at least $1 - \epsilon$, respectively.
We take $\delta = 2^{-r} \epsilon$ for the second term.
Then, for any $G, H$ with $d(G, H) \le \delta$, we have $d_{TV}(z_r(G), z_r(H)) \le \epsilon$;
hence, by the optimal coupling theorem,\footnote{See~\citep{cuestaalbertos1993optimal}, or \url{pages.uoregon.edu/dlevin/AMS_shortcourse/ams_coupling.pdf} (May, 2021).} there exists a coupling between $Z_r(G)$ and $Z_r(H)$ such that $P(Z_r(G) \neq Z_r(H)) = d_{TV}(z_r(G), z_r(H))  \le \epsilon$.
Thus, the output of the algorithm $\mathcal{A}$ coincides on $G$ and $H$ with a probability at least $1 - \epsilon$.
Therefore we have $|f(G) - f(H)| \le 3 \epsilon$ with probability at least $1 - 3 \epsilon$. 
By taking the expectation, we obtain the result.

Suppose $f$ is uniformly continuous.
For any $\epsilon > 0$, let $\delta > 0$ be the corresponding constant in the continuity definition.
Take a $\delta/2$-net $\{ H_1, \dots, H_C \}$ and let $r = \lceil \log 4/\delta \rceil$.
The algorithm $\mathcal{A}$ performs random sapling to estimate the distribution $(z_1(G), \dots, z_r(G))$ with accuracy $\delta / 2$ with probability at least $1 - \epsilon$.
Then, it outputs $f(H_j)$, where $H_j$ is the nearest neighborhood of $G$.
By the construction, the algorithm finds $H_j$ with $d(G, H_j) \le \delta$ with probability at least $1 - \epsilon$.
Therefore, we have $|f(G) - \mathcal{A}(G)| = |f(G) - f(H_j)| \le \epsilon$ with probability as least $1 - \epsilon$.
\end{proof}

It should be emphasized that the space of all graphs equipped with the randomized Benjamini--Schramm topology is \emph{not} compact, because there is a continuous but not uniformly continuous function; the average degree function is such an example.

\subsection{Proofs for Applications}
\label{appendix:proofs-app}

This section provides the proofs for the theoretical applications section in the main part (Section~\ref{sec:applications}).
Most notably, the proofs for Theorem~\ref{thm:rademacher-complexity},~\ref{thm:est-size-generalizable}, and~\ref{thm:domain-adaptation} are provided here.

\begin{proof}[Proof of Proposition~\ref{prop:perturbation}]
Let $M$ be the endpoints of the random edges.
Then, $M$ induces a uniform distribution on the vertices of $G$.
Any run with $Z_r(G) \cap M = \emptyset$ can be coupled with $Z_r(G')$; so the coupling probability is
\begin{align}
P(M \cap Z_r(G) = \emptyset) 
&= \sum_{x \in M} P(x \not \in Z_r(G)) \\
&\le |M| r^r / n \label{eq:M-ineq} \\ 
&= 2 r^r \delta.
\end{align}
By the optimal coupling theorem, we have $d_{TV}(z_r(G), z_r(G')) =  2 r^r \delta$.
Therefore, 
\begin{align}
    d(G, G') 
    &= \sum_{r=1}^\infty 2^{-r} d_{TV}(G_r, G_r') \\
    &\le \sum_{r=1}^\infty 2^{-r} \min \{ 1, 2 r^r \delta \} \\
    &\le s^s \delta + 2^{-s}
\end{align}
for any $s$.
By putting $s = \log \log (1/\delta)$, we obtain the result.
\end{proof}
If $M$ is chosen adversarially, we cannot obtain the inequality \eqref{eq:M-ineq}.
In particular, if $M$ contains a vertex $x$ with a large PageRank, as the probability of $x \in Z_r(G)$ is large, we cannot bound the distance.

\begin{proof}[Proof of Theorem~\ref{thm:rademacher-complexity}]
We use the following inequality that bounds the Rademacher complexity by the covering number $C_\mathcal{F}(\epsilon)$ of the function space $\mathcal{F}$:
\begin{align}
    R_n(\mathcal{F}) \le \inf_{\epsilon > 0} \left\{ \epsilon + O\left( \sqrt{ \frac{\log C_{\mathcal{F}}(\epsilon)}{n} } \right) \right\}.
\end{align}

We choose an $\epsilon/2$-net of the graphs of size $C(\epsilon/2)$ by Lemma~\ref{lem:totally-bounded-epsilon-net}.
Then, we define an (external) $\epsilon$-cover of the space of $1$-Lipschitz functions by the piecewise constant functions whose values are discretized by $\epsilon/2$, where the pieces are the Voronoi regions of the $\epsilon$-net;
it is easy to verify this is an $\epsilon$-cover of the space of $1$-Lipschitz functions.
This shows $C_{\mathcal{F}}(\epsilon) \le (2 / \epsilon)^{C(\epsilon/2)} = 2^{2^{2^{2^{O(\log (1/\epsilon) \log \log (1/\epsilon))}}}}$. 
By substituting $\epsilon$ satisfying $O(\log (1/\epsilon) \log \log (1/\epsilon)) = \log \log \log (n / \log n))$, we obtain the result.

$\log C_\mathcal{F}(\epsilon) = 2^{2^{2^{O( \log 1/\epsilon \log \log 1/\epsilon})}}$.
We set $\epsilon$ to be $O(\log 1/\epsilon \log \log 1/\epsilon) = \log \log \log (n / \log n)^2$.
Then, by definition, $\sqrt{\log C_\mathcal{F}(\epsilon) / n} = 1/\log n$. 

We try to evaluate $\epsilon$. 
We see $\epsilon$ satisfies $\log 1/\epsilon \log \log 1/\epsilon = \Omega(\log \log \log n)$.

Recall that $x \log x = y$ iff $y = e^{W(x)}$ where $W(x)$ is the Lambert W function.
Since $W(x) = \log (x / \log x) + \Theta(\log \log x / \log x)$, we have $y \ge x / \log x$.
By using this formula, we have $\log 1/\epsilon = \Omega(\log \log \log n / \log \log \log \log n)$.
\end{proof}

\begin{proof}[Proof of Theorem~\ref{thm:est-size-generalizable}]
By Theorem~\ref{thm:estimability-equiv-continuity}, an estimable function $f$ is uniformly continuous. 
Let $\delta$ be the constant for $\epsilon$ for the continuity.
By Lemma~\ref{lem:totally-bounded-epsilon-net}, there is an $\delta$-net $\{ H_1, \dots, H_C \}$ and let $N(\delta)$ be the maximum number of vertices in the graphs in the $\delta$-net.
Our algorithm $\mathcal{A}$ outputs $\mathcal{A}(G) = f(H_j)$ where $H_j$ is the nearest neighbor of $G$.
This algorithm achieves the accuracy of $\epsilon$ because $d(G, H_j) \le \delta$.
Also, the algorithm can be constructed only accessing graphs of size at most $N = \max_j |V(H_j)|$. 
Hence, $f$ is size-generalizable.
\end{proof}

\begin{proof}[Proof of Lemma~\ref{lem:transfer-bound}]
This is an adaptation of \cite{shen2018wasserstein} to our metric space. 
Their proof only uses the ``easy'' direction of the Kantorovich--Rubinstein duality, which holds on any metric space.
Hence, we obtain this lemma.

To be self-contained, we will give a proof.
For any $1$-Lipschitz function $f$ and any coupling $\pi$ between $\mathcal{D}_1$ and $\mathcal{D}_2$, we have the following ``easy'' direction of the Kantorovich--Rubinstein duality:
\begin{align}
    \E_1[f(G_1)] -
    \E_2[f(G_2)] 
    &= \E_{(G_1, G_1) \sim \pi} \E[ f(G_1) - f(G_2) ] \\ 
    &\le \E_{(G_1, G_2) \sim \pi} \E[ d(G_1, G_2) ] \\ 
    &\le W(\mathcal{G}_1, \mathcal{G}_2).
\end{align}
By putting $f = (h - h') / 2$, we obtain
\begin{align}
    \E_1 | h(G) - h'(G) | - \E_2 | h(G) - h'(G) | \le 2 W(\mathcal{G}_1, \mathcal{G}_2).
\end{align}
Hence,
\begin{align}
    \E_1 | y - h(G) | 
    &\le \E_1 | y - h'(G) | + \E_1 | h(G) - h'(G) | \\
    &=  \E_1 | y - h'(G) | + \E_1 [ h(G) - h'(G) | + \E_2 | h(G) - h'(G) | - \E_2 | h(G) - h'(G) | \\
    &\le \E_1 | y - h'(G) | + \E_2 | h(G) - h'(G) | + 2 W(\mathcal{G}_1, \mathcal{G}_2) \\
    &\le \E_2 | y - h(G) | + \E_1 | y - h'(G) | + \E_2 | h(G) - h'(G) | + 2 W(\mathcal{G}_1, \mathcal{G}_2).
\end{align}
By taking the infimum over $h'$, we obtain the theorem.
\end{proof}

\begin{proof}[Proof of Theorem~\ref{thm:domain-adaptation}]
We obtain the result by combining Theorem~\ref{thm:rademacher-complexity} and Lemma~\ref{lem:transfer-bound}.
\end{proof}

\begin{proof}[Proof of Proposition~\ref{prop:size-gen}]
Let $\mathcal{G}_N$ be the distribution of random $d$-regulra graphs of size $N$.
Then, we can see that $Z_r(G_N) \mid G_N \sim \mathcal{G}_N$ has no cycle with probability at least $1 - r^{O(r)}/N$.
This implies that we can couple $Z_r(G) \mid G \sim \mathcal{G}$ and $Z_r(G_{\le N}) \mid G_{\le N} \sim \mathcal{G}_{\le N}$ with probability at least $1 - r^{O(r)}/N$.
By putting $r = \log 1/\epsilon$, we obtain the result.
\end{proof}

\begin{proof}[Proof of Proposition~\ref{prop:size-gen2}]
This follows from the proof of Lemma 10.31 and Exercise 10.31 in \cite{lovasz2012large}.
\end{proof}

\begin{proof}[Proof of Proposition~\ref{prop:wasserstein}]
We can choose $N$ by the maximum number of vertices in the $\epsilon$-net.
Then, we have $d(G, \Pi(G)) \le \epsilon$. 
Thus the Wasserstein distance is bounded by $\epsilon$.
\end{proof}

\begin{proof}[Proof of Theorem~\ref{thm:indist}]
We introduce a helper concept, \emph{$\equiv$-indistinguishably estimable}, which is a class of functions that is estimable by $\equiv$-indistinguishable computation on $r$-profile.
By the same argument as Theorem~\ref{thm:universal-rbs-gnn}, we can show that RBS-GNN[$\equiv$] can represent $\equiv$-indistinguishably estimable function.

Now we prove that if a function $f$ is estimable and $\equiv$-indistinguishable, then it is $\equiv$-indistinguishably estimable. 
Because $f$ is estimable, there exists $\delta > 0$ such that $|f(G) - f(H)| \le \epsilon / 2$ if $d(G, H) < \delta$.
We fix a $\delta$-net $\{ H_1, \dots, H_C \}$ of the graph space.
We consider a quotient space of the graphs by $\equiv$, select a representative $[G]$ to each quotient, and assign $H_i$ to $[G]$ which is the nearest neighbor of the representative $G$.

Our estimator $\mathcal{A}$ is the following.
First, we obtain a subgraph $S$ by sampling sufficiently many vertices to be $d(S, G) \le \delta$ with probability at least $1 - \epsilon$.
Second, we take the representative $[S]$ of the equivalent class containing $S$.
Finally, we output the value $f(H)$, where $H$ is the nearest neighbor of $[S]$.
By construction, $\mathcal{A}$ is an $\equiv$-indistinguishable computation after the sampling.
Here,
\begin{align}
    |f(G) - \mathcal{A}(G)| \le |f(G) - f(S)| + |f(S) - f([S])| + |f([S]) - f(H)| \le \epsilon
\end{align}
with probability at least $1 - \epsilon$,
where the first term in the right-hand side is at most $\epsilon/2$ with probability at least $1 - \epsilon$ due to the sampling and continuity, the second term is zero due to the $\equiv$-indistinguishability, and the last term is at most $\epsilon/2$ due to the $\delta$-net and uniform continuity.
\end{proof}

\section{Extended Results: Finite-Dimensional Vertex Features}
\label{appendix:feat}

We can extend our framework to the vertex-featured case.
We assume the vertex features are in $[0, 1]^d$.
This assumption is well-aligned with the pre-processing step in practice where vertex features are normalized~\citep{gcn,nt2019revisiting,hoang2020graph}. 

The estimability is defined similarly, where we additionally assume that the estimation is uniformly continuous with respect to the vertex features on the sampled subgraph (in the standard topology of $\mathbb{R}^d$).
Then, we can prove the RBS-GNN can estimate arbitrary estimable vertex-featured graph parameters.

The difficulty is how to define the topology on the vertex-featured graphs.
As in the non-featured case, We want to define $Z_r(G)$ by the ``frequency'' of the graphs.
However, since there are uncountably many vertex-featured graphs, we need a technique.
To address this issue, we fix an $\epsilon$-net on $[0, 1]^d$; it has the cardinality of $(1/\epsilon)^d$.
Then, we approximate the vertex features of the sampled graph by the elements of the $\epsilon$-net by the uniform continuity. 
Then, the number of ``vertex-featured graphs'' of $N$ vertices is bounded by $((1 / \epsilon)^d)^{O(N \times N)}$; hence we can define the randomized Benjamini--Schramm topology.
The space is totally bounded since the $\epsilon$-net is constructed by combining the $\epsilon$-net of the graph and the $\epsilon$-net of $[0, 1]^d$.
Note that this makes no significant difference on the size of the $\epsilon$-net since the difference is absorbed in the nested power.

%
%
%

\section{Extended Results: Vertex Classification Problems}
\label{appendix:vertex}

In this section, we extend our framework for the graph classification problem to the vertex classification problem.

The vertex classification problem is usually defined as follows. 
We are given a set of graphs $G_1, \dots, G_N$ with the ``supervised vertices'' $S_1 \subseteq V(G_1), \dots, S_N \subseteq V(G_N)$ and the labels $y_u$ on the supervised vertices.
The task is to find an equivariant function $h \colon G \mapsto (y_1, \dots, y_n) \in \mathcal{Y}^{V(G)}$.
This formulation, however, is not suitable for large graphs because it needs to output values to all the vertices.
Here, we recognize a vertex classification problem as a rooted graph classification problem as in \cite{nt2019revisiting}: 
The input of the problem is a set of pairs $(y_{v_i}, (G_i, v_i))$ of rooted graphs $(G_i, v_i)$ and the label $y_{v_i}$. 
The goal is to find a function $h$ such that $h((G, v)) \approx y_v$.
It should be noted that a graph $G$ with a supervised nodes $S \subseteq V(G)$ in the original formulation is transformed to $|S|$ rooted graphs $\{ (G, v) : v \in S \}$.

Our framework for the graph classification problem is easily extended to the rooted graph classification problem.
We first modify our computational model by assuming the root of the graph is available at the beginning of the computation.
Then, the estimability of the function is defined in the same way using this computational model.
Then, we modify the RBS-GNN to have one additional ball centered at the root vertex, i.e.,
\begin{align}
    \text{\normalfont RBS-GNN}((G, v)) = g\left(f(C_0), \sum_{j=1}^{N_C} f(C_j)\right)
\end{align}
where $B_0, \dots, B_k$ are the random balls obtained by \textsf{RandomBallSample} where the root of $B_0$ is conditioned by $v$, and $C_0, \dots, C_{CS}$ are the weakly connected components of $G[B_0 \cup \dots \cup B_k]$ where $C_0$ contains $v$.
We can prove that any estimable vertex parameter in the random neighborhood model for the rooted graph is estimable using the RBS-GNN.
Also, by extending the randomized Benjamini--Schramm topology to the rooted graphs, we see the estimability coincides with the uniform continuity in the randomized Benjamini--Schramm topology of rooted graphs.

Using this topology, we can obtain the vertex classification version of the results in Section~\ref{sec:applications}. 
As the number of rooted graphs of $N$ vertices is $N$ times larger than the number of non-rooted graphs of $N$ vertices, the covering number of the rooted graph space is larger than that of the non-rooted graph space.
But this makes no significant difference because this gap is absorbed in the nested logarithm.

This formulation gives several consequences on the vertex classification problem.
\begin{itemize}
\item 
We can evaluate the required number of supervised vertices to obtain the desired accuracy. 
In this formulation, each supervised vertex $v$ corresponds to a rooted graph $(G, v)$.
Thus, if the supervised vertices are chosen randomly, the required number of supervised vertices is evaluated by the Rademacher complexity of the model.
In particular, we can obtain a model with an accuracy $\epsilon$ from the constantly many supervised vertices.
\item 
We can characterize the difficulty of a vertex classification problem with different supervision using transfer learning.
Imagine a situation that the supervised vertices have large degrees, but we want to predict the vertex property on low-degree vertices.
This situation can be recognized that the training and test rooted graph distributions, $\mathcal{G}_\text{train}$ and $\mathcal{G}_\text{test}$, are different.
Therefore, we can apply Lemma~\ref{lem:transfer-bound} to obtain an estimation of the test error, which involves the Wasserstein distance of these distributions.
\end{itemize}

\section{Extended Results: Practical Applications}
\label{appendix:exp}

Many real-world graph parameters are estimable in our framework.
This section provides a review of graph parameters and their estimability in our random neighborhood model.
Most notably, this section demonstrates the usage of Theorem~\ref{thm:estimability-equiv-continuity} for practical graph parameters and provide the proof for Proposition~\ref{prop:graphsage-algo2-clustering-coeff}.
In addition, we provide experimental results on real-world datasets to verify our theoretical claims.

\subsection{Graph Parameters}
For convenience, we re-state the random neighborhood model and the \emph{estimable} definitions here.
The original definitions were provided in Section~\ref{ssec:computational-model}.

\begin{definition}[Random Neighborhood Model]
\label{def:rnm}
The random neighborhood computational model allows the following three queries given an input graph $G$:
{\normalfont
\begin{itemize}[noitemsep,topsep=0pt]
    \item \tsf{SampleVertex}($G$): Sample a vertex $u \in V$ uniformly randomly.
    \item \tsf{SampleNeighbor}($G$, $u$): Sample a vertex $v$ from the neighborhood of $u$ uniformly randomly, where $u$ is an already obtained vertex.
    \item \tsf{IsAdjacent}($G$, $u$, $v$): Return whether the vertices $u$ and $v$ are adjacent, where $u$ and $v$ are already obtained vertices.
\end{itemize}}
This computational model induces an estimability definition and a topology, named random Benjamini-Schramm, on the graph space $\mathcal{G}$. 
\end{definition}

\begin{definition}[Constant-Time Estimable Graph Parameter]
\label{def:est_restate}
A graph parameter $p$ is constant-time estimable on the random neighborhood model (estimable for short) if for any $\epsilon > 0$ there exists an integer $N$ and a randomized algorithm $\mathcal{A}$ in the random neighborhood model such that $\mathcal{A}$ performs at most $N$ queries and $|\mathcal{A}(G) - p(G)| < \epsilon$ with probability at least $1 - \epsilon$ for all graphs $G \in \mathcal{G}$. 
\end{definition}

\subsubsection{Non-estimable Graph Parameters}

We first see that an estimable parameter is bounded as follows.

\begin{proposition}
\label{prop:boundedness}
An estimable parameter $p$ is bounded.
\end{proposition}
\begin{proof}
Recall that an estimable parameter is uniformly continuous (Theorem~\ref{thm:estimability-equiv-continuity}).
Let $\delta$ be the constant for $\epsilon = 1$ for the continuity.
Let us fix an $\delta$-net $\{ G_1, G_2, \dots \}$ of the graph space using the totally boundedness of the space (Lemma~\ref{lem:totally-bounded-epsilon-net}).
Then, $p$ is bounded by $C = 1 + \max_i p(G_i)$.
In fact, for any graph $G \in \mathcal{G}$, we have
\begin{align}
    p(G) \le | p(G) - p(G_i)| + |p(G_i)| \le C
\end{align}
where $G_i$ is the nearest neighbor of $G$ in the $\epsilon$-net.
\end{proof}

Example~\ref{example:nonest} in Section~\ref{ssec:computational-model} states that the number of vertices and the min/max degrees are not estimable;
these immediately follow form Proposition~\ref{prop:boundedness} since they are unbounded parameters.
Similarly, the average degree is unbounded so it is not estimable.
The connectivity function is an example of bounded but non-estimable graph parameter.

\begin{proposition}
\label{prop:connectivity-impossible}
$p(G) = 1[ \text{$G$ is connected} ]$ is not estimable.
\end{proposition}
\begin{proof}
We prove this proposition by giving a counter example showing $p$ violates the definition for continuity.
Since continuity is equivalent to estimability, such counter example would also disprove the estimability of $p$.

We first fix $\epsilon = 1/2$. 
Then, we choose two graphs $G_1$ and $G_2$ such that $G_1$ is the disjoint union of two cliques of size $N$ and $G_2$ is obtained from $G_1$ by adding one edge between them.
The figure below demonstrates for the case $N=6$.
\begin{figure}[h]
    \centering
    \includegraphics[width=0.6\textwidth]{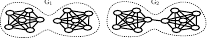}
\end{figure}

We see that $d(G_1, G_2) \le \delta$ for any chosen $\delta > 0$ if $N$ is sufficiently large.
This is because the distribution $z_r(G_1)$ and $z_r(G_2)$ of isomorphism classes only different at the event that one of the two connected vertices of $G_2$ is sampled.
The probability for such event becomes increasingly insignificant when $N$ is sufficiently large.
Hence, the distance $d(G_1, G_2)$ can be arbitrarily small as stated above.
However, by the definition of the connectivity function, for any $N$ we always have $|p(G_1) - p(G_2)| = 1 > 1/2$.
Hence, $p$ is not continuous, and by Theorem~\ref{thm:estimability-equiv-continuity}, it is also not estimable.
\end{proof}

\subsubsection{Estimable Graph Parameters}

The following propositions proves statements in Example~\ref{example:est}.

\begin{proposition}
\label{prop:triangle}
The triangle density is a uniformly continuous parameter and estimable.
\end{proposition}

Using Theorem~\ref{thm:estimability-equiv-continuity}, it is clear that we only need to prove estimability or continuity. 
We show both proofs for this case as a demonstration.
For simplicity, we assume the input graph $G$ is undirected.

\begin{proof}[Proof for Estimability]
We show that the triangle density is estimable by constructing a random algorithm and prove that this algorithm estimates the triangle density to an arbitrary precision dependent only on the number of random samples (Definition~\ref{def:est}).
The randomized algorithm can be implemented under the random neighborhood computational model (Definition~\ref{def:rnm}).

\begin{algorithm}
\caption{Triangle Density Estimation in the Random Neighborhood Model}
\label{algo:triangle-algo}
\begin{algorithmic}[1]
\Procedure{IsTriangle}{$G$, $u$, $v$, $q$}
    \State $uv \leftarrow$ \tsf{IsAdjacent}($G$, $u$, $v$);
    \State $uq \leftarrow$ \tsf{IsAdjacent}($G$, $u$, $q$);
    \State $qv \leftarrow$ \tsf{IsAdjacent}($G$, $q$, $v$);
    \State{\Return $uv \wedge uq \wedge qv$}; \Comment{$\wedge$ is the logical ``and''.}
\EndProcedure

\Procedure{TriangleDensity}{$G$, $T$} 
    \State{triangles $\leftarrow 0$};
    \For{$i$ in $1,\dots,T$}
    \State $u, v, q \leftarrow$ \tsf{SampleVertex}($G$); \Comment{Runs 3 times to get 3 samples.}
    \State{triangles $\leftarrow$ triangles $+$ \tsf{IsTriangle}($G, u, v, q$)};
    \EndFor
    \State{\Return $\frac{1}{T}\text{triangles}$};
\EndProcedure
\end{algorithmic}
\end{algorithm}

The procedure \tsf{IsTriangle} in Algorithm~\ref{algo:triangle-algo} return $1$ if the three input vertices induce a triangle and $0$ otherwise.
Let $X$ be the output of \tsf{IsTriangle} given three random vertices $u, v, \text{and}\ q$ from graph $G$, and $\bar{X}$ be the output of \tsf{TriangleDensity}.
By definition, the expectation $\mathbb{E}(X)$ is the true triangle density $p_\Delta$.
Since the $p_\Delta$ and $\text{Var}(X)$ are clearly finite, we can apply the Chebyshev's concentration bound to the sample average $\bar{X}$ with $T$ samples to obtain the following. For any $\epsilon > 0$,
\begin{align}
    \mathbb{P}(|\bar{X} - p_\Delta)| \geq \epsilon) \leq \frac{\text{Var}(X)}{\epsilon^2 T}.
\end{align}
This bound shows that if we take $T = O(\epsilon^{-3})$ samples, then with probability at least $1-\epsilon$ we obtain an estimation less than $\epsilon$ from the true value.
Note that $T$ is only dependent on the precision $\epsilon$ and not the size of $G$; this shows the intuition behind the constant-time nature of our RBS-GNN.

\end{proof}
\begin{proof}[Proof for Continuity]
We now prove that the triangle density $p_\Delta$ is uniformly continuous in the randomized Benjamini-Schramm topology.
For a given $\epsilon > 0$, we choose $\delta = 2^3 \epsilon$.
Let $G_1$ and $G_2$ be two graphs satisfying $d(G_1, G_2) \le \delta$. 
We denote two random variable $X_1$ and $X_2$ to represent the event that random sampling from $G_1$ and $G_2$ obtained a triangle.
$X_1$ and $X_2$ follows $z_\Delta(G_1)$ and $z_\Delta(G_2)$ distributions, respectively.
By the optimal coupling theorem, the random sampling on $G_1$ and $G_2$ can be coupled with probability at least $1 - \epsilon$.
\begin{align}
    d_{TV}(z_\Delta(G_1), z_\Delta(G_2)) = \min_{(X_1, X_2)-\text{couplings}}\mathbb{P}(X_1 \neq X_2)
\end{align}

Hence, by the definition of the triangle density, these differs at most $\epsilon$.
\end{proof}

Using a similar technique, we can prove the estimability or equivalently uniformly continuity of the local clustering coefficient.

\begin{proposition}
\label{prop:clustering}
The local clustering coefficient is uniformly continuous and estimable.
\end{proposition}

\subsection{Graph Classification in Random Neighborhood Model}
\label{appendix:clf}

We show the results for RBS-GNN on social networks datasets in the TUDatasets repository~\citep{Morris+2020}: COLLAB, REDDIT-BINARY, and REDDIT-MULTI5K.
We preprocess these datasets in the same way as proposed by~\citet{gin}. 
Because of this pre-processing, each vertex has a feature vector representing its position in the degree distribution. 
The reason for such setting is because in 1-WL, the degree determines the initial coloring~\citep{morris2020weisfeiler}. 
A summary of the datasets is given in Table~\ref{tab:datainfo}.

\begin{table}[h]
\begin{center}
\caption{Overview of the graph classification datasets. This is a small part of the TUDataset~\citep{Morris+2020}. $|\mathcal{G}|$ denotes the total number of graphs in the dataset, $\overline{\nu}(G)$ denotes the average number of nodes per graph, $\vert c \vert$  denotes the number of classes, $d$ denotes the dimensionality of vertex features (created by~\citep{gin}). $r$ denotes the radius of random balls and also the branching factor. $k$ denotes the number of random balls. $\%$ denotes the average coverage of random balls in terms of the number of edges. $\%$ Memory denotes the relative data storage size.}
\label{tab:datainfo}
\begin{tabular}{lrrrrrrrrr}
\toprule
\textsc{Datasets} & $|\mathcal{G}|$ & $\overline{\nu}(G)$ & $|c|$ & $d$ & $r$ & $b$ & $k$ & $\% |E(G)|$ & $\%$ Memory\\
\midrule
COLLAB & 5000 & 74.5 & 3 & 367 & 2 & 5 & 3 & $58.3\pm 22.9$ & $55.7$\\
RDT-BINARY & 2000 & 429.6 & 2 & 566 & 3 & 5 & 3 & $37.3\pm 28.6$ & $14.9$\\
RDT-MULTI5K & 5000 & 508.5 & 5 & 734 & 3 & 5 & 3 & $23.8\pm 17.7$ & $14.2$\\ 
\bottomrule
\end{tabular}
\end{center}
\end{table}

To simulate the random ball sampling procedure of RBS-GNN, we pre-sample the original datasets and use these random samples in both training and testing.
Table~\ref{tab:datainfo} shows the sampling setting we used to report the results in Table~\ref{tab:cv_scores}.
We prepared multiple other settings for $r$, $b$, and $k$, see the provided source code for more detail (\texttt{supp/notebooks/Preprocessing.ipynb}).
Let $N_{G}(\cdot)$ be the neighborhood function of the sampled input graph, we construct a $K$-layers $f$ as follows.
\begin{align}
    h_G^{(\ell)}(u) &= \text{MLP}^{(\ell)}\left(\sum_{v \in N_G(u)} h^{(\ell-1)}(v)\right), \ell = 1,\dots,K \\
    f(G) &= \sum_{\ell=1,\dots,K} \sum_{u \in G} h_G^{(\ell)}(u),
\end{align}
where $\text{MLP}^\ell$ is a single layer MLP with no activation.
Let $g$ be a 2-layers MLP with ReLU activation functions, the final output of $\text{RBS-GNN}[\equiv_2]$ is given similar to Equation~\eqref{eq:rb-gnn}:
\begin{align}
    \text{RBS-GNN}[\equiv_2](G) = g\left(\sum_j f(G)\right).
\end{align}

The implemented RBS-GNN is denoted by $\text{RBS-GNN}[\equiv_2]$ because its GNN component $f$ resembles a message-passing GNN such as GIN~\citep{gin}. 
In all our experiments, the graph neural network $f$ has 4 propagation layers and 5 MLP layers, each of these layers have 32 ReLU hidden units (see \texttt{supp/src/rbsgnn/models/mpgnn\_batched.py} and \texttt{supp/notebooks/Cross-Validation Scores.ipynb}).
Regularization methods are weight decay ($10^{-3}$), learning rate decay (initialized at $0.01$, step size 50, $\gamma=0.5$), and dropout (0.5).
%

\paragraph{Computational Resources} We run all our experiments on a single computer having a single Intel CPU (i7-8700K\@3.70GHz), 64GB DDR4 memory, and a NVIDIA GeForce GTX 1080Ti GPU with CUDA 11.3 (driver version 465.31).
The system runs Linux Kernel 5.12.6.
Our model's prototype is implemented using Python 3.9 and PyTorch 1.8.1+cu111 (see \texttt{supp/src/requirements.txt} for the detail of the Python environment).

\paragraph{Cross-Validation Scores}
Reporting the 10-folds (also 3-folds and 5-folds) cross-validation scores for graph learning model is a common task in the literature~\cite{gin,gntk,wu2020comprehensive}.
We compare our practical implementation of RBS-GNN to existing benchmarks in terms of 10-folds cross validation scores.
We show the results for other baselines reported by~\citet{gin}.
These baselines includes WL-subtree, PatchySan, and AWL (see Section 7 in~\citep{gin} for more detail).
Note that RBS-GNN only has access to partial inputs for both training and testing procedures. The fractions of observed edges and storage memory are shown in Table~\ref{tab:datainfo}.

\begin{table}[h]
\begin{center}
\caption{Best test accuracy (in percentage) for the graph classification task. Note that RBS-GNN only has access to 20$\sim$60 percent of edges and 15$\sim$50 percent of node features (Table~\ref{tab:datainfo}).}
\label{tab:cv_scores}
\begin{tabular}{lcccc}
\toprule
\textsc{Models} & COLLAB & RDT-BINARY & RDT-MULTI5K\\
\midrule
GIN-0 & $80.2 \pm 1.9$ & $92.4 \pm 2.5$ & $57.5 \pm 1.5$\\
WL subtree & $78.9 \pm 1.9$ & $81.0 \pm 3.1$ & $52.5 \pm 2.1$ \\
PatchySan & $72.6 \pm 2.2$ & $86.3 \pm 1.6$ & $49.1 \pm 0.7$ \\
AWL & $73.9 \pm 1.9$ & $87.9 \pm 2.5$ & $ 54.7 \pm 2.9$ \\
\midrule
$\text{RBS-GNN}[\equiv_2]$ & $80.3 \pm 1.5$ & $79.0 \pm 1.9$ & $44.0 \pm 1.4$ \\
\bottomrule
\end{tabular}
\end{center}
\end{table}

The result in Table~\ref{tab:cv_scores} shows that at best we can achieve similar result for COLLAB while observing only 55.7\% of the data.
Note that this experiment is different from any random pooling or drop-out techniques because we use random balls for \emph{both} train and test.
The results for REDDIT datasets are also quite similar to the results of complete-observation models. 
As shown in Table~\ref{tab:datainfo}, we only observe about 14\% of the REDDIT original datasets.
We selected such extreme example to show that although by a small observation, in some cases GNNs can still predict well.
This observation implies that the true labeling function of these dataset is smooth in the random Benjamini-Schramm topology.  

\paragraph{Size Generalization}
Our Theorem~\ref{thm:est-size-generalizable} identified size-generalizability with estimability.
We verify this theoretical result experimentally as follows.
For each dataset, we split into two sets of train and test data (0.5/0.5 split).
Furthermore, the train set consists of smaller graphs while the test set has larger ones.
We then report the test accuracy for targets being the global clustering coefficient (estimable) and the max-degree (non-estimable).
In machine learning, it is often more beneficial to work with classification rather than regression; therefore, we categorized the clustering coefficients and the max-degree values into 5 classes.
Each class represents a 20-percentile proportion of the data (see \texttt{supp/notebooks/preprocessing.ipynb}).
The test accuracy is reported in Table~\ref{tab:size_generalization}.

\begin{table}[h]
\begin{center}
\caption{Test accuracy (in percentage) of $\text{RBS-GNN}[\equiv_2]$ for the size-generalization task.}
\label{tab:size_generalization}
\resizebox{\textwidth}{!}{
\begin{tabular}{ccc|ccc}
\toprule
\multicolumn{3}{c}{\textit{Global Clustering}} & \multicolumn{3}{c}{\textit{Max-Degree}}\\
COLLAB & RDT-BINARY & RDT-MULTI5K & COLLAB & RDT-BINARY & RDT-MULTI5K \\
\midrule
$23.7 \pm 3.7$ & $33.7 \pm 1.1$ & $40.0 \pm 2.0$ & $16.1 \pm 2.1$ & $43.4 \pm 1.2$ & $39.8 \pm 3.1$ \\
\bottomrule
\end{tabular}
}
\end{center}
\end{table}

\end{appendix}

\end{document}